%% file: E2E_CRO_arxiv_version.tex
\documentclass[11pt]{article}

\usepackage{amsmath, amsfonts, amsthm, amssymb}
\usepackage{authblk, color, bm, graphicx}
\usepackage{booktabs, multicol, multirow}

\usepackage{xspace}
\usepackage[left=1in,top=1.2in,right=1in,bottom=1.2in]{geometry}

\usepackage[english]{babel} 
\usepackage{natbib} 
\usepackage{array} 
\usepackage{lscape} 
\usepackage{adjustbox} 

\usepackage{multirow}
\usepackage{booktabs}
\usepackage{makecell} 
\usepackage{caption}
\usepackage{threeparttable}
\usepackage{authblk}
\usepackage{hyperref}
\usepackage{microtype}
\usepackage{graphicx}
\usepackage{booktabs} 
\usepackage{natbib} 
\usepackage{mathtools} 
\usepackage{booktabs} 
\usepackage{tikz} 

\input{math_commands.tex}

\usepackage{algorithmic}
\usepackage{pgfplots}
\usepackage{hyperref}
\usepackage{rotating}
\usepackage{booktabs}

\usepackage{url}
\usepackage{amsmath}
\usepackage{amssymb}
\usepackage{mathtools}
\usepackage{amsthm}
\usepackage{float}
\usepackage{dsfont}
\usepackage[capitalize,noabbrev]{cleveref}
\usepackage{algorithm}
\usepackage{algpseudocode}
\theoremstyle{plain}
\newtheorem{theorem}{Theorem}[section]

\newtheorem{lemma}[theorem]{Lemma}

\theoremstyle{definition}
\newtheorem{definition}[theorem]{Definition}

\theoremstyle{remark}
\newtheorem{remark}[theorem]{Remark}

\usepackage{pdfpages}
\usepackage[utf8]{inputenc} 
\usepackage[T1]{fontenc}    
\usepackage{hyperref}       
\usepackage{url}            
\usepackage{booktabs}       
\usepackage{amsfonts}       
\usepackage{nicefrac}       
\usepackage{microtype}      
\usepackage{xcolor}         
\usepackage{mathtools}
\usepackage{amsmath,amsfonts,amssymb}
\usepackage{algorithm}
\usepackage{algpseudocode}
\usepackage{graphicx}
\usepackage{subfig}
\usepackage{tikz}
\usepackage{amsthm}
\usepackage{multirow}
\usepackage{graphicx}
\usepackage{mwe}
\usepackage{array}
\usepackage{xcolor}
\usepackage{pgfplots}

\usepackage{dblfloatfix}
\usepackage{pgfplots}
\pgfplotsset{compat=1.17}

\usepackage{pgf}
\usepackage{pgfplots}
\usetikzlibrary{calc}
\usepackage{dblfloatfix}

\definecolor{bblue}{HTML}{1f77b4}
\definecolor{rred}{HTML}{d62728}
\definecolor{ggreen}{HTML}{2ca02c}
\definecolor{oorange}{HTML}{ff7f0e}

	\definecolor{green}{rgb}{0.16, 0.67, 0.53}
\definecolor{blue}{rgb}{0.19, 0.55, 0.91}
\definecolor{red}{rgb}{0.8, 0.25, 0.33}
\definecolor{orange}{rgb}{0.91, 0.45, 0.32}

\makeatletter
\newcommand\footnoteref[1]{\protected@xdef\@thefnmark{\ref{#1}}\@footnotemark}
\makeatother
\usepackage{graphicx}
\usepackage{xcolor}

\DeclareRobustCommand{\legendsquare}[1]{%
  \textcolor{#1}{\rule{1ex}{1ex}}%
}

\newcommand{\ACMod}[1]{{#1}}

\def\Expect{{\mathbb E}}
\def\Prob{{\mathbb P}}

\def\min{\mathop{\rm min}}
\def\max{\mathop{\rm max}}
\def\argmax{\mathop{\rm argmax}}
\def\argmin{\mathop{\rm argmin}}

\def\sup{\mathop{\rm sup}}
\def\inf{\mathop{\rm inf}}

\newcommand{\V}[1]{{{\boldsymbol #1}}}
\newcommand{\x}{\V{x}}
\newcommand{\X}{\mathcal{X}}

\renewcommand{\1}{\V{1}}

\newcommand{\mymbox}[1]{\mbox{\scriptsize #1}}
\renewcommand{\Re}{\mathbb{R}}
\newcommand{\quoteIt}[1]{``#1''}

\newcommand{\network}{{\mathfrak{F}}}

\definecolor{lgreen}{HTML}{78C357}
\definecolor{dgreen}{HTML}{006400}
\definecolor{lblue}{HTML}{7BA9D0}
\definecolor{dblue}{HTML}{00008B}
\definecolor{newred}{HTML}{FF0000}
\definecolor{lred}{HTML}{FFC0CB}
\definecolor{salmon}{HTML}{FFA07A}
\definecolor{grey}{HTML}{808080}

\newcommand{\mxi}{{m}}

\newcommand{{\polx}}{{\boldsymbol{x}}}
\newcommand{\U}{{\mathcal{U}}}

\newcommand{\Dxipsi}{{\mathcal{D}_{\psi \xi}}}

\newcommand{\ECRO}[0]{TbS}
\newcommand{\EECRO}[0]{DTbS}
\newcommand{\CROES}[0]{{ETO-DbS}}
\newcommand{\CROCS}[0]{ETO-CPS}
\newcommand{\CROCCS}[0]{ETO-ACPS}

\renewcommand{\eqref}[1]{(\ref{#1})}

\newcommand{\removed}[1]{{}}

\tikzstyle{arrow} = [thick,->,>=stealth]
\tikzstyle{startstop} = [rectangle, rounded corners, 
minimum width=3cm, 
minimum height=1cm,
draw=black]
\tikzstyle{end} = [rectangle, rounded corners, 
minimum width=3cm, 
minimum height=1cm,
draw=black]
\vspace{0.25cm}

\tikzstyle{arrow} = [thick,->,>=stealth]
\tikzstyle{startstop} = [rectangle, rounded corners, 
minimum width=1cm, 
minimum height=1cm,
draw=black]
\tikzstyle{end} = [rectangle, rounded corners, 
minimum width=1cm, 
minimum height=1cm,
draw=black]
\vspace{0.25cm}

\usepackage[textsize=tiny]{todonotes}

\usepackage{hyperref}

\usepackage{amsmath}
\usepackage{amssymb}
\usepackage{mathtools}
\usepackage{amsthm}

\usepackage[capitalize,noabbrev]{cleveref}

\usepackage[textsize=tiny]{todonotes}

\begin{document}
\title{End-to-end Conditional Robust Optimization
}
\author[1]{Abhilash Chenreddy\footnote{Corresponding author\\  Email addresses: erick.delage@hec.ca (Erick Delage), abhilash.chenreddy@hec.ca (Abhilash Chenreddy)}}
\author[2]{Erick Delage}
\affil[1,2]{GERAD and Department of Decision Sciences, HEC Montr\'eal, 
Canada}

\date{\today}

\maketitle

\begin{abstract}
The field of Contextual Optimization (CO) integrates machine learning and optimization to solve decision making problems under uncertainty. Recently, a risk sensitive variant of CO, known as Conditional Robust Optimization (CRO), combines uncertainty quantification with robust optimization in order to promote safety and reliability in high stake applications. Exploiting modern differentiable optimization methods, 
we propose a novel end-to-end approach to train a CRO model in a way that accounts for both the empirical risk of the prescribed decisions and the quality of conditional coverage of the contextual uncertainty set that supports them. While guarantees of success for the latter objective are impossible to obtain from the point of view of conformal prediction theory, high quality conditional coverage is achieved empirically by ingeniously employing a logistic regression differentiable layer within the calculation of coverage quality in our training loss.  
We show that the proposed training algorithms produce decisions that outperform the traditional estimate then optimize approaches.
\end{abstract}

\section{Introduction}

In a standard Machine Learning (ML) setting, $\Psi \subseteq \mathbb{R}^m$ represent the input set and $\Xi \subseteq \mathbb{R}^m$ represent the output sets and we aim to learn a model $\network_\theta$  parameterized by $\theta$ that approximates the relationship between the input and output sets. In real-world applications, we usually have a dataset of $M$ samples, $\Dxipsi := \{( \psi_i, \xi_i)\}_{i=1}^M$ which are used to approximate the underlying input-output relationship learned by the model. For a new data sample $\psi \in \Psi$, the model trained on $\Dxipsi$ is used to predict a corresponding target $\xi = \network_{\theta}(\psi) $. Recently, there has been a growing interest in integrating this estimation process with the subsequent optimization process. In this context, the prediction is used within a cost minimization problem $\hat{x}^*(\psi):=\arg\min_{x\in\X}c(x,\network_\theta(\psi))$, where $\X\subseteq \Re^n$ is the set of feasible decisions and $c(x,\xi)$ the cost function. The intent is to produce an adapted decision with low out-of-sample expected cost $\Expect[c(\hat{x}^*(\psi),\xi)]$. 
 When there a mismatch between the predictive loss $\mathcal{L}$ and the cost function $c(x,\xi)$, a small error in estimating $\xi$ for a given $\psi$ can lead to highly suboptimal $x^*(\psi)$ (see \cite{elmachtoub2022smart}). Task-based (or decision-focused) learning (c.f. \cite{mandi2023decisionfocused,donti2017task}) addresses this issue by training the model $\network_{\theta}$ directly on the performance of the policy $x^*(\psi)$. By trading off predictive performance in favour of task performance, the task-based approach can give near optimal decisions.

In high stakes applications, a DM usually demonstrates of a certain degree of risk aversion by requiring some level of protection against a range of plausible future scenarios. A natural risk averse variant of integrated ML and optimization takes the form of Conditional Robust Optimization (see \cite{chenreddy2022data}), which integrates conformal prediction with robust optimization. Specifically, machine learning is first used to produce a contextually adapted uncertainty set $
\U(\psi)$ known to contain with high probability the realized $\xi$, which is then inserted to the conditional robust optimization model:
\begin{align}
    \label{eq:CRO}
    x^*(\psi) := \arg\min_{x \in \mathcal{X}} \max_{\xi \in \mathcal{U}(\psi)} c(x, \xi),
\end{align}

To this date, the methods proposed in the literature follow an Estimate Then Optimize (ETO) paradigm. Namely, data is first used to calibrate the contextual uncertainty set. 
This set is then used as an input to the CRO problem to get the adapted robust decision $x^*(\psi)$. However, this uncertainty set calibration process does not consider the downstream optimization task which can lead to misalignment between the initial estimation loss function and the robust optimization objectives.

In this paper, we propose a novel end-to-end learning framework for conditional robust optimization that constructs the contextual uncertainty set by accounting for the downstream task loss. Our contributions can be described as follows:
\begin{itemize}
    \item We propose for the first time an end-to-end training algorithm to produce contextual uncertainty sets that lead to reduced risk exposure for the solution of the down-stream CRO problem.
    \item We introduce a novel joint loss function aimed at enhancing the conditional coverage of the contextual uncertainty sets $\mathcal{U}(\psi)$ while improving the CRO objective
    \item We demonstrate through a set of synthetic environments that our end-to-end approach surpasses ETO approaches at the CRO task while achieving comparable if not superior conditional coverage with its learned contextual set.
    \item We show empirically how our end-to-end learning approach outperforms other state-of-the-art methods on a portfolio optimization problem using the real world data from the US stock market.    
\end{itemize}

\section{Related work}

\textbf{Estimate Then Optimize} Popularized by the initial work of \cite{hannah2010nonparametric} is a framework that integrates ML and optimization tasks. Several approaches are proposed to learn the conditional distribution from data. \cite{kannan2023residuals, sen2018learning} propose using residuals from the trained regression model to learn conditional distributions. \cite{bertsimas2020predictive} approach assigns weights to the historical observations of the parameters and solves the weighted SAA problem. Besides the CSO problems, There has been a growing interest in integrating ML techniques in Robust Optimization problems.\cite{chenreddy2022data} identifies clusters of the uncertain parameters based on the covariate data and calibrates the sets for these clusters. \cite{patel2023conformal} propose using non-convex prediction regions to construct uncertainty sets.\cite{blanquero2023contextual} constructs contextual ellipsoidal uncertainty sets by making normality assumptions.
\cite{ohmori2021predictive} uses non-parametric kNN model to identify the minimum volume ellipsoid to be used as an uncertainty set. \ACMod{\cite{sun2024predict} solves a robust contextual LP problem where a prediction model is first learned, then uncertainty is calibrated to match robust objectives.}
It is to be noted that all these CRO approaches follow the ETO paradigm.

\textbf{End-to-end learning} is a more recent stream of work that integrates the Estimation and Optimization tasks and trains using the downstream loss.
\cite{donti2017task} proposed using an end-to-end approach for learning probabilistic machine learning models using task loss. 
\cite{elmachtoub2022smart} learns contextual point predictor by minimizing the regret associated with implementing prescribed action based on the mean estimator.
\cite{amos2017optnet} uses implicit differentiation methods to train an end-to-end model. \cite{butler2023efficient} solves large-scale QPs using the ADMM algorithm that decouples the differentiation procedure for primal and dual variables. \cite{elmachtoub2022smart,mandi2020smart} propose using a surrogate loss function to train integrated methods to address loss functions with non-informative gradients. \ACMod{
\cite{wang2023learning} proposes learning a non-contextual uncertainty set by maximizing the expected performance across a set of  randomly drawn parameterized robust constrained problems while ensuring  guarantees on the probability of constraint satisfaction with respect to the joint distribution over perturbance and robust problems. \cite{costa2023distributionally} proposes a distributionally robust end-to-end system that integrates point prediction and  robustness tuning to the portfolio construction problem.}
We refer the reader to \cite{sadana2023survey} for a broader discussion on both ETO and end-to-end models.\\
\textbf{Uncertainty quantification} methods are employed to estimate the confidence of deep neural networks over their predictions \cite{ kontolati2022survey}. Common UQ approaches include using Bayesian methods like stochastic deep neural networks, ensembling over predictions from several models to suggest intervals and models that directly predict uncertain intervals. \cite{gawlikowski2021survey}. Beyond estimating predictive uncertainty, ensuring its statistical reliability is crucial for safe decision-making \cite{guo2017calibration}. Conformal prediction has become popular as a distribution-free calibration method \cite{shafer2008tutorial}. Although conformal prediction ensures marginal coverage, attaining conditional coverage in the most general case is desirable  \cite{vovk2012conditional}. Although considered unfeasible \cite{romano2020malice} offers group conditional guarantees for disjoint groups by independently calibrating each group. 

\section{Estimate then Robust Optimize}
\label{sec:ETO}
The concept of \quoteIt{estimate then optimize}(ETO) comes from the contextual optimization literature (see \cite{sadana2023survey}). In this framework, the role of the \textbf{Estimation} process is to 
quantify the uncertainty about $\xi$ given the observed $\psi$. This is given as input to an \textbf{Optimization} problem that prescribes an optimal contextual decision $x^*(\psi)$.

When the downstream optimization problem is a CRO problem, the estimation step is required to produce a region that adapts to the observed covariates $\psi$ and is expected to contain the response $\xi$ with high confidence. This can be done indirectly by first calibrating a conditional distribution model $F_\theta(\psi)$ to the data, followed by an implied confidence region $\U_\theta(\psi)$ that satisfies $\Prob_{F_\theta(\psi)}(\xi\in\U_\theta(\psi))=1-\epsilon$. For e.g., when one assumes that $\xi |\psi \sim \mathcal{N}(\hat{\mu}(\psi), \hat{\Sigma}(\psi))$, one can learn $(\hat{\mu}(\psi), \hat{\Sigma}(\psi))$ by maximizing the log-likelihood function (see \cite{barratt2023covariance})
\[-\frac{n}{2} \log(2\pi) + \sum_{j=1}^n \log L(\psi)_{jj} - \frac{1}{2} \| L(\psi)^\top (\xi - \hat{\nu}(\psi)) \|_2^2\]
where $L(\psi)$ and $\hat{\nu}(\psi)$ are the parametric mappings that can be used to compose  
$\hat{\mu}(\psi) := L(\psi)^\top\nu(\psi)$ and $\hat{\Sigma}(\psi) = L(\psi)^{-\top}L(\psi)^{-1}$. Using the $\alpha$ quantile from the chi-squared distribution with $m$ degrees of freedom, one can define $\U_\theta(\psi)$ that satisfies $\Prob(\xi\in\U_\theta(\psi))=1-\epsilon$ asymptotically.

Some recent work completely circumvents the need for the $F_\theta$ intermediary by calibrating some $\U_\theta(\psi)$ directly on the dataset. For e.g. \cite{chenreddy2022data} propose identifying a k-class classifier, $a : \Re^m \rightarrow [K]$ to reduce $\mathcal{U}_{\theta}(\psi) := \mathcal{U}_{\theta}(a(\psi)) $ such that $\Prob(\xi\in\U_\theta(k)|a(\psi) = k) \geq 1-\epsilon \,\forall k$. The literature on conformal prediction belongs to the second type and separates the calibration of the shape of $\U(\psi)$ from the calibration of its size, parametrized by a radius $r>0$, on a reserved validation set in order to provide out-of-sample marginal coverage guarantees of the form $\Prob(\xi\in\U(\psi))\geq 1-\epsilon,$ where the probability is taking over both the draw of the validation set and of the next sample.

\section{End-to-End Conditional Robust Optimization} \label{sec:E2E}
While the ETO approach presented in the section \ref{sec:ETO} presents an efficient way to conditionally quantify the uncertainty, it does not take into account the quality of the decisions $x^*(\psi)$ that is prescribed by the downstream CRO model. 
%
In practice, the quality of a robust decision is usually assessed by measuring the risks associated to the cost produced on a new data sample (a.k.a. out-of-sample). We  assume  that this risk is measured by a risk measure that reflect the amount of risk aversion experienced by the DM. For instance, one can use conditional value-at-risk with $\rho_\alpha(X):=\inf_t t+ (1/(1-\alpha))E[(X-t)^+]$, which computes the expected value in the right tail of the random cost and covers both expected value and the worst-case cost as special cases (i.e. with $\alpha=0$ and $1$ respectively). Motivated by recent evidence (see \cite{elmachtoub2022smart}) indicating that performance improvement can be achieved by employing a decision-focused/task-based learning paradigm, we propose end-to-end conditional robust optimization.

\subsection{The  ECRO training problem}

Formally, we let $\Psi\subseteq\Re^m$ be an arbitrary support set for $\psi$ whereas $\Xi\subseteq\mathbb{R}^m$ 
is assumed for simplicity to be contained within a ball centered at 0 of radius $R_\xi$. We consider $c(x,\xi)$ to be convex in $x$ and concave in $\xi$ and let $\mathcal{X}(\psi):=\{x\in\Re^n|g(x,\psi)\leq 0,\,h(x,\psi)=0\}$ be a convex feasible set for $x$, possibly dependent on $\psi$, and defined through a set of convex inequalities, identified using $g:\Re^n\times\Re^m\rightarrow \Re^J$ and affine equalities, identified using an affine mapping $h:\Re^n\times\Re^m\rightarrow \Re^J$. The conditional optimal policy then becomes:
\begin{align}
    \label{eq:CRO_new}
    x^*(\psi, \U) := \arg\min_{x \in \mathcal{X}(\psi)} \max_{\xi \in \mathcal{U}(\psi)} c(x, \xi),
\end{align}
where we make explicit how the decision depends on both the contextual uncertainty set and the realized covariate. Given a parametric family of contextual uncertainty set $\U_\theta$ with $\theta\in \Theta$ and a dataset $D_{\psi\xi}:=\{(\psi^i,\xi^i)\}_{i=1}^{M}$, the ECRO training problem consists in identifying 
\begin{align}\label{eq:E2Eloss}
\min_{\theta\in\Theta} \mathcal{L}_{ECRO}(\theta) := \rho_{i\sim M}(c(x^*(\psi^i,\mathcal{U}_\theta),\xi^i)),
\end{align}
where for simplicity we assume $\rho(\cdot)$ to be a conditional value-at-risk measure, and $\U_\theta(\psi)$ to be ellipsoidal for all $\psi$. Namely, we can assume that 
\begin{align}\label{eq:ConSet}
    &\mathcal{U}_\theta(\psi) =\mathcal{E}(\mu_\theta(\psi),\Sigma_\theta(\psi), r)\\ \nonumber
    &\;:= 
    \{\ \xi\in\mathbb{R}^{\mxi} : (\xi - \mu_\theta(\psi))^T \Sigma_\theta(\psi)^{-1}(\xi - \mu_\theta(\psi)) \leq 1\}\,,
\end{align}
for some $\mu_\theta: \Re^m\rightarrow \Re^m$ and $\Sigma_\theta:\Re^m\rightarrow \mathcal{S}_+$, where $\mathcal{S}_+$ is the set of positive definite matrices, for all $\theta\in\Theta$. While the robust optimization literature suggests various uncertainty set structures that facilitates resolution of the RO problem,  the ellipsoidal set stands out as a natural one to employ as it retains numerical tractability (see \cite{ben1998robust}) and can easily be described to the DM.

\begin{figure}[!ht]
    \centering
    \tikzset{
      basic/.style  = {draw, text width=5em, drop shadow,  rectangle},
      root/.style   = {basic, thin, align=center,
                   fill=gray!45 , text width=5em},
      level 2/.style = {basic, thin,align=center, fill=gray!30,
                   text width=5em},
      level 3/.style = {basic, thin, align=left, fill=gray!20, text
      width=5em, node distance = 40pt} 
    }
    \begin{tikzpicture}[node distance=2cm]
    \begin{scope}[scale=0.65, transform shape]
\node (b1) [end] {Estimation};
\node (b2) [startstop, below of=b1] {Optimization};
\node (b3) [startstop, below of=b2]{Task loss};
\draw [arrow] (0,1.5) -- (b1) 
    node [midway,right] {$\Dxipsi$} node [midway,below] {};
\draw [arrow] (b1) -- (b2) node [midway,right] {$\U_{\theta}$};
\draw [arrow] (b2) -- (b3) node [midway,right] {$x^*(\cdot,\U_\theta)$};
\draw [arrow] (b3) -- (0,-5.5) node [midway,right] {$\U_{\theta^*},\,x^*(\cdot,\U_{\theta^*})$};
\draw [thick] (b3) -- (-2, -4);
\draw [thick] (-2,-4) -- (-2, 0) node [midway,left] {$\nabla_\theta \mathcal{L}_{ECRO}(\theta)$}; 
\draw [arrow] (-2,0) -- (b1);
    \end{scope}
\end{tikzpicture}\\
    \caption{Training pipeline for task-based learning  }
    \label{fig:TSpipeline}
    \end{figure}
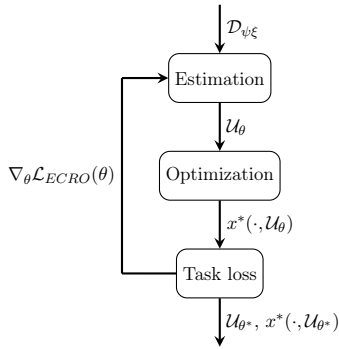

The training pipeline for the tasked-based learning approach is illustrated in figure \ref{fig:TSpipeline}. In this pipeline, one starts from an arbitrary $\theta^0$, the optimization problem \eqref{eq:CRO_new} is solved first for each data point, the resulting optimal actions are then implemented in order to measure the empirical risk under $D_{\psi\xi}$, which we call empirical ECRO loss of $\theta^0$. A gradient of $\mathcal{L}_{ECRO}(\theta)$ can then be used to update $\theta^0$ in a direction of improvement. 
Key steps in in this pipeline consists in computing $x^*(\psi^i,\U_{\theta})$ efficiently and in a way that enables differentiation with respect to $\theta$.

\subsection{Reducing and solving the robust optimization task}

Given the convex-concave structure of $c(x,\xi)$ and the convexity and compactness of the ellipsoidal set, we can employ Fenchal duality (see \cite{ben2015deriving}) to reformulate the min-max problem as a simpler minimization form over an augmented decision space. Specifically, we first replace the original cost function with the equivalent cost
\[\bar{c}(x,\xi):=\left\{\begin{array}{cl}c(x,\xi) & \mbox{if $\|\xi\|_2\leq R$}\\ -\infty & \mbox{otherwise}\end{array}\right.,\]
which integrates information about the domain of $\xi$.
One can then employ theorem 6.2 of \cite{ben2015deriving}, to show that problem \eqref{eq:CRO} can be reformulated as:
\begin{align}
\min_{x \in \mathcal{X}(\psi),v} f(x,v,\psi):=\delta^*(v|\mathcal{U}_\theta(\psi)) - \bar{c}_*(x, v)\label{eq:CROmin}
\end{align}
where the support function
\begin{align}
\delta^*(v|\mathcal{U}_{\theta}(\psi)) := \sup_{\xi \in \mathcal{U}_\theta(\psi)} \xi^Tv =  \mu^Tv+  \sqrt{v^T \Sigma^{-1} v}
\end{align}
while the partial concave conjugate function is defined as
\begin{align*}
\bar{c}_*(x, v) := \inf_{\xi}  v^T \xi - \bar{c}(x, \xi) =  \inf_{\xi : \|\xi\|_2\leq R}  v^T \xi - c(x, \xi)
\end{align*}
This leads to $x^*(\psi,\mathcal{U}(\psi))$ being the minimizer of the convex minimization problem: 
\begin{equation}\label{eq:xopt}
\min_{x \in \mathcal{X}(\psi),v} f(x,v,\psi)  
\end{equation}
with $f(x,v,\psi):=\mu^Tv+ \sqrt{v^T \Sigma^{-1} v} - \bar{c}_*(x, v)$,
a jointly convex function of $x$ and $v$ and finite valued over its domain, and with sub-derivatives:
\begin{align*}
&\nabla_x f(x,v,\psi)=\mu+(1/\sqrt{v^T \Sigma^{-1} v})v-\xi^*(x,v) \\
&\nabla_v f(x,v,\psi)=\nabla_xc(x,\xi^*(x,v)),
\end{align*}
%
where $\xi^*(x,v):=\argmin_{\xi : \|\xi\|_2\leq R}  v^T \xi - c(x, \xi)$. Revisiting the procedure outlined in Figure \ref{fig:TSpipeline}, one can observe that the training process requires a forward pass to find the optimal solutions and a backward pass to iteratively update the parameter vector $\theta$. This requires the computation of the gradients of the solution to the problem \ref{eq:E2Eloss} with respect to the input parameters which are passed through the reformulated CRO problem. 
Furthermore, the minimization procedure in problem \ref{eq:E2Eloss} entails navigating through the risk measure $\rho$. 
These aspects will be further explored in the next section.

\subsection{ Gradient for problem \eqref{eq:E2Eloss}}

In training problem \eqref{eq:E2Eloss}, the gradient of $\mathcal{L}_{ECRO}(\theta)$ with respect to $\theta$ can be obtained using the chain rule:
\begin{align*}
\allowdisplaybreaks
\nabla_\theta &\mathcal{L}_{ECRO}(\theta)= \sum_i \frac{\partial\rho_{i\sim M}(y_i)}{\partial y_i}\big|_{y_i=c(x^*(\psi^i,\mathcal{U}_\theta),\xi^i) } \cdot
 \nabla_{\x} c(\x)\big |_{\x=x^*(\psi^i,\mathcal{U}_\theta)} \cdot  \\
&\quad\left(\nabla_\mu x^*(\psi^i,\mathcal{E}(\mu,\Sigma_\theta(\psi^i)))\big|_{\mu=\mu_\theta(\psi^i)}\nabla_\theta\mu_{\theta}(\psi^i)\right.
\left.+ \nabla_\Sigma x^*(\psi^i,\mathcal{E}(\mu_\theta(\psi^i),\Sigma))\big|_{\Sigma=\Sigma_\theta(\psi^i)}\nabla_\theta\Sigma_{\theta}(\psi^i)\right).
\end{align*}
Based on \cite{doi:10.1137/1.9781611976595.ch6}, 
when $\rho(Y):= \mbox{CVaR}_\alpha(Y)$, one can employ the subdifferential:
\[\nabla_{\V{y}}\rho_{i\sim M}(y_i):=\V{\upsilon}(\V{y})\]
with $\V{\upsilon}(\V{y})\in\argmax_{\V{\upsilon}\in\Re^M_+:\1^T \V{\upsilon}=1, \V{\upsilon}\leq ((1-\alpha)N)^{-1}} \V{\upsilon}^T \V{y}$.

Given that $\nabla_{\x} c(\x)$, $\nabla_\theta\mu_{\theta}(\psi)$, and $\nabla_\theta\Sigma_{\theta}(\psi)$ can be readily obtained using Auto-Differentiation \cite{seeger2017auto} when $c(\x)$, $\mu_{\theta}(\psi)$, and $\Sigma_{\theta}(\psi)$ are differentiable, we focus the rest of this subsection on the process of identifying $\nabla_{(\mu,
\Sigma)}x^*(\psi,\mathcal{E}(\mu,\Sigma))$. Following the decision-focus learning literature (see \cite{blondel2022efficient}), one can identify such derivatives by exploiting the fact that any optimal primal dual pair $(x^*, v^*, \lambda^*, \nu^*)$ of problem \eqref{eq:xopt} must satisfy the Karush-Kuhn-Tucker (KKT) conditions, which take the form:
\begin{align*}
&G(x^*, v^*, \lambda^*, \nu^*, \mu, \Sigma,\psi) =0,\qquad g(x^*,\psi) \leq 0, \lambda^* \geq 0 .
\end{align*}
where
\begin{align*}
&G(x^*, v^*, \lambda^*, \nu^*,\mu,\Sigma,\psi) := \begin{bmatrix} \nabla_x f(x^*,v^*,\psi ) + \nabla_x g(x^*, \psi)^T \lambda^* + \nabla_x h(x^*, \psi)^T \nu^* \\ \lambda^* \circ g(x^*,\psi) \\ h(x^*,\psi) \end{bmatrix}
\end{align*}
and \(\circ\) denotes the Hadamard product of two vectors.

One can therefore apply implicit differentiation to the constraints $G(x^*, v^*, \lambda^*, \nu^*,\mu,\Sigma,\psi)=0$ to identify $\nabla_{(\mu,\Sigma)} x^*(\psi,\mathcal{E}(\mu,\Sigma))$ simultaneously with the derivatives of $v^*$, $\lambda^*$, and $\nu^*$ with respect to the pair $(\mu,\Sigma)$. 
Specifically, one is required to solve the system of equations:
\begin{align*}
&\frac{\partial}{\partial x, v, \lambda, \nu} G(x^*,v^*, \lambda^*, \nu^*, \mu,\Sigma,\psi) \cdot 
\frac{\partial}{\partial (\mu,\Sigma)} (x^*, v^*, \lambda^*, \nu^*)(\mu,\Sigma) = - \frac{\partial}{\partial (\mu,\Sigma)} G(x^*, v^*,\lambda^*, \nu^*, \mu,\Sigma,\psi),
\end{align*}
where \(\frac{\partial}{\partial (x, v, \lambda, \nu)} G\) denotes the Jacobian of the mapping \(G\) with respect to \((x, v, \lambda, \nu)\). We refer to \cite{blondel2022efficient} and \cite{duvenaud2020deep} for further details on the computations of related to implicit differentiation.

\subsection{Task-based Set (\ECRO{}) Algorithm}
In this section, we delve into implementation details of the ECRO training pipeline. Regarding the contextual ellipsoidal set $\mathcal{E}(\mu_\theta(\psi),\Sigma_\theta(\psi))$, we follow the ideas proposed in \cite{barratt2023covariance} and employ a neural network that maps from $\network_\theta:\Re^m \rightarrow \Re^m\times \Re^{m(m+1)/2}\times\Re$. The first set of outputs is used to define $\mu_\theta(\psi)$ while the second and third set forms a lower triangular matrix $L_\theta(\psi)$ and scalar $r_\theta(\psi)$, which is made independent of $\psi$ w.l.o.g., used to produce $\Sigma_\theta(\psi):=r_\theta(\psi)L_\theta(\psi)L_\theta(\psi)^T$. The positive definiteness of $\Sigma_\theta(\psi)$ is ensured by taking an exponential in the last layer of the network for the output that appear in the diagonal of $L$.
The architecture of the neural network can be found in appendix \ref{architecture}. 

The second set of notable details have to do with solving for $x^*(\psi^i, \mathcal{E}(\mu^i_{\theta},\Sigma^i_{\theta},r_{\theta}))\ \forall i$. In our implementation of end-to-end learning for conditional robust optimization, we found that a trust region optimization (TRO) method \cite{byrd2000trust} could efficiently solve the reformulated robust optimization problem \eqref{eq:xopt} and provide primal dual solution pairs for this problem. Given that each episode of the training would pass through the same set of data points, we further observed that the training accelerated significantly (see figure in Appendix \ref{fig:covcomp}) when the trust region was interrupted early (after $K=5$ iterations) as long as it would be warm started at the solution found at the previous epochs. Algorithm \ref{alg:ecrotraining} presents our proposed training framework for the ECRO approach.

\begin{algorithm}
  \caption{ECRO Training with Trust Region Solver}\label{alg:ecrotraining}
  \begin{algorithmic}[1]
  \State \textbf{input}: dataset $\Dxipsi$, max epochs $T$, max TRO steps $K$, batch size $N$, protection level $\alpha$
  \State Initialize a warm start  buffer $\{\bar{x}_1,\dots,\bar{x}_M\}$ with each $\bar{x}^i\in\mathcal{X}(\psi_i)$
  \State Initialize network parameters $\theta$ and $t=1$
      \While{not converged and ($t \le T$)}
        \State Sample a batch of $N$ indices $\mathcal{B}\subset \{1,\dots,M\}$
         \For{$i\in \mathcal{B}$}
         \State //Run TRO for up to $K$ steps
        \State $x_i^t, \lambda_i^t, \nu_i^t \gets$ \textsc{TRO}($\bar{x}_{i}$, $\mu_{\theta}(\psi_{i}),\Sigma_\theta(\psi_{i}), K)$
        \State $\bar{x}_i\gets x_i^t$ \Comment{Update warm start}
               
      \EndFor
        \State \textbf{compute} $\mathcal{L}_{ECRO}(\theta)$ and $\nabla_\theta\mathcal{L}_{ECRO}(\theta)$ for $i\sim\mathcal{B}$
        \State $\theta \gets \theta - \mbox{step size}\cdot\nabla_\theta\mathcal{L}_{ECRO}(\theta)$ 
        
      \EndWhile\label{euclidendwhile}
      
      \State \textbf{return} $\theta$
  \end{algorithmic}
\end{algorithm}

\section{End-to-End CRO with Conditional Coverage}

Recall that the ETO framework summarized in section \ref{sec:ETO} focused on producing contextual uncertainty set with appropriate marginal coverage (of $1-\epsilon$) of the realization of $\xi$. The training pipeline in section \ref{sec:E2E} was at the other end of the spectrum, disregarding entirely the objective of coverage to increase task performance. In practice, coverage can be a heavy price to pay to obtain performance as it implies a loss in the explainability of the prescribed robust decision. It is becoming apparent that many DM suffer from algorithm aversion (see \citep{burton2020systematic}) and could be reluctant to implementing a robust decision produced from an ill covering uncertainty set. 

We further argue that traditional ETO might already face resistance to adoption given the type of coverage property attributed to the ETO sets, i.e. $\Prob(\xi\in\U(\psi))=1-\epsilon$. Indeed, 
marginal coverage guarantees only hold in terms of the joint sampling of $\psi$ and $\xi$. This implies that it offers no guarantees regarding the coverage of $\xi$ given the observed $\psi$ for which the decision is made. In fact, a 90\% marginal coverage can trivially be achieved using $\Xi$ when $\psi\in\Psi$ and $\emptyset$ otherwise, as long as $\Prob(\psi\in\Psi)=1-\epsilon$. This is clearly an issue for applications with critical safety considerations and motivates 
seeking conditional coverage in addition to the marginal coverage when designing $\mathcal{U}(\psi)$. In this section, we outline a training procedure that integrates a sub-procedure that enhances the conditional coverage performance. 


\subsection{The conditional coverage training problem}

We start by briefly formalizing the difference between the two types of coverage in the definition below. 
\begin{definition}
    Given a confidence level $1-\epsilon$, a contextual uncertainty set mapping $\mathcal{U}(\cdot)$ is said to satisfy \textbf{marginal coverage} if $\Prob(\xi\in\mathcal{U}(\psi))=1-\epsilon$, and to satisfy \textbf{conditional coverage} if $\Prob(\xi\in\mathcal{U}(\psi)|\psi)=1-\epsilon$ almost surely.
\end{definition}
The following lemma identifies a necessary and sufficient condition for a contextual set to satisfy conditional coverage.
\begin{lemma}\label{thm:lossEquiv}
    A contextual uncertainty set $\U(\psi)$ satisfies conditional coverage, at confidence $1-\epsilon$, if and only if
   \[\mathcal{L}_{\mymbox{CC}}(\theta):=\Expect[\,\left(\Prob(\xi \in \U(\psi)|\psi)-(1-\epsilon)\right)^2\,]=0\]
\end{lemma}

\begin{proof}
For any random variable $X$, one can show that :
\begin{align*}
X&=1-\epsilon \mbox{ a.s} \\
&\;\Rightarrow\;\Expect[(X-(1-\epsilon))^2]=1\cdot(1-\epsilon-(1-\epsilon))^2=0
\end{align*}    
\qquad \quad and that, since $y^2\leq 0 \Leftrightarrow y=0$,
\begin{align*}
\Expect&[(X-(1-\epsilon))^2]=0\\
&\;\Rightarrow\; (X-(1-\epsilon))^2 = 0 \mbox{ a.s.} \;\Rightarrow\;X = 1-\epsilon \mbox{ a.s.}.
\end{align*}    
\qquad \quad By letting $X:=\Prob(\xi\in\U_\theta(\psi)|\psi)$, we obtain our result.
\vskip -0.2in
\end{proof}
Equipped with Lemma \ref{thm:lossEquiv}, we formulate the \quoteIt{theoretical} conditional coverage training problem as $ \min_{\theta\in\Theta}\;\;\mathcal{L}_{CC}(\theta)$.
Since the true conditional distribution $\Prob(\xi\in\U_\theta(\psi)|\psi)$ is typically inaccessible to the DM, we propose an approximation that will make  $\mathcal{L}_{CC}(\theta)$ practical.

\subsection{Regression-based Conditional Coverage Loss}

Given a set $\U$, one can define a  binary random variable $y(\psi,\xi,\U) := \1\{\xi \in \U(\psi)\}$, and rewrite the conditional probability distribution $\Prob(\xi\in\U(\psi)|\psi)$ as $\Prob(y(\psi,\xi,\U) = 1|\psi)$. Using the i.i.d sample data in $\Dxipsi$, one can approximate this conditional probability using a parametric model, i.e.  $\Prob(y(\psi,\xi,\U) |\psi)\approx g_{\phi}(\psi)$ for some $\phi\in\Phi$. The parameters $\phi$ can be calibrated by minimizing the negative conditional log-likelihood of  $\{y(\psi^i,\xi^i,\U)\}_{i=1}^M$:
\begin{align} \label{eq:NLL}
    \phi^*(\U):=\arg\min_{\phi\in\Phi} -\frac{1}{M} \sum_{i=1}^{M} \log g_{\phi}(\psi^i)^{y^i}(1-g_{\phi}(\psi^i))^{1-y^i},
\end{align}
where $y_i:=y(\psi^i,\xi^i,\U)$.
Using the parametric approximation $g_{\phi^*(\U)}(\psi) \approx \Prob(\xi \in \U(\psi)|\psi)$ and replacing the unknown true distribution of $(\psi,\xi)$ with the empirical one, we obtain our regression-based conditional coverage loss function 
\begin{align*}\label{eq:coverageLoss}
\hat{\mathcal{L}}_{CC}(\theta) := \mathbb{E}^{\mathcal{\Dxipsi}}[(g_{\phi^*(\U_\theta)}(\psi) - (1-\epsilon))^{2}].
\end{align*}

The gradient of $\hat{\mathcal{L}}_{CC}(\theta)$ can be obtained using similar decision-focused training methods as employed for $\mathcal{L}_{ECRO}(\theta)$ given that:
\begin{align*}
\nabla_\theta&\hat{\mathcal{L}}_{CC}=\sum_{i=1}^M 2(g_{\phi^*(\U_\theta)}(\psi^i) - (1-\epsilon))\nabla_\phi g_{\phi^*(\U_\theta)}(\psi^i)\cdot 
\sum_{j=1}^M \partial \phi^*(\mathcal{E}(\mu,\Sigma_\theta(\psi^i))) / \partial y^j\cdot \\
&\left(\nabla_\mu y^j(\psi^j,\xi^j,\mathcal{E}(\mu,\Sigma_\theta(\psi^j)))\big|_{\mu=\mu_\theta(\psi^j)}\nabla_\theta\mu_{\theta}(\psi^j)\right.
\left.+ \nabla_\Sigma y^j(\psi^j,\xi^j,\mathcal{E}(\mu_\theta(\psi^j),\Sigma))\big|_{\Sigma=\Sigma_\theta(\psi^j)}\nabla_\theta\Sigma_{\theta}(\psi^j)\right),
\end{align*}
where the main challenges reside again in the step of differentiating through the minimizer of problem \eqref{eq:NLL}.

\subsection{Dual Task based Set (\EECRO{}) algorithm}
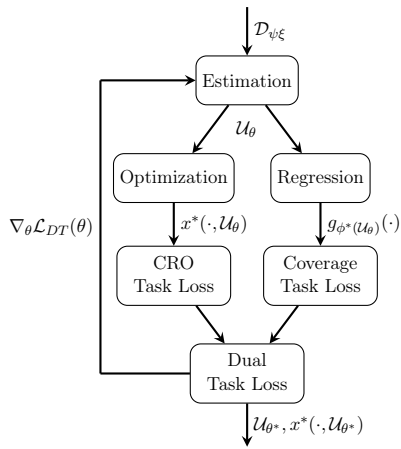
\begin{figure}[!ht]
    \centering
    \tikzset{
      basic/.style  = {draw, text width=5em, drop shadow,  rectangle},
      root/.style   = {basic, thin, align=center,
                   fill=gray!45 , text width=5em},
      level 2/.style = {basic, thin,align=center, fill=gray!30,
                   text width=5em},
      level 3/.style = {basic, thin, align=left, fill=gray!20, text
      width=5em, node distance = 40pt} 
    }
    \begin{tikzpicture}[node distance=2cm]
    \begin{scope}[scale=0.65, transform shape]
\node (b1) [end] {Estimation};
\node (b2) [startstop, below of=b1, xshift= -1.5cm] {Optimization};
\node (b3) [startstop, below of=b2]{\begin{tabular}{c} CRO \\ Task Loss \end{tabular}};
\node (b5) [startstop, right of=b2, xshift= 1cm] {Regression};
\node (b6) [startstop, right of=b3, xshift= 1cm]{\begin{tabular}{c} Coverage \\ Task Loss \end{tabular}};
\node (b7) [startstop, below of=b3, xshift= 1.5cm]{\begin{tabular}{c} Dual \\ Task Loss \end{tabular}};
\draw [arrow] (0,1.5) -- (b1) 
    node [midway,right] {$\Dxipsi$} node [midway,below] {};
\draw [arrow] (b1) -- (b2) node [midway,right, xshift= 0.4cm] {$\U_{\theta}$};
\draw [arrow] (b1) -- (b5);
\draw [arrow] (b5) -- (b6) node [midway,right]  {$g_{\phi^*(\U_\theta)}(\cdot)$};

\draw [arrow] (b2) -- (b3) node [midway,right] {$x^*(\cdot,\U_\theta)$};
\draw [arrow] (b3) -- (b7) node [midway,left] {};
\draw [arrow] (b6) -- (b7);
\draw [thick] (b7) -- (-3, -6);
\draw [thick] (-3,-6) -- (-3, 0) node [midway, left] {$\nabla_\theta \mathcal{L}_{DT}(\theta)$}; 
\draw [arrow] (-3,0) -- (b1);
\draw [arrow] (b7) -- (0,-7.5) node [midway, right] {$\U_{\theta^*}, x^*(\cdot,\U_{\theta^*})$};
   \end{scope}
\end{tikzpicture}\\
    \caption{Training pipeline for dual task based learning  }
    \label{fig:TSCovpipeline}\vskip -0.1in
    \end{figure}
We conclude this section with the presentation of our novel integrated algorithm that learns the contextual uncertainty set network $\network_\theta$ by incorporating both the risk mitigation and conditional coverage tasks in the training. 
Indeed our \EECRO{} training algorithm minimizes the following double task loss function that trades off between the two task objectives: 
\begin{equation}
\mathcal{L}_{DT}(\theta) = \gamma \mathcal{L}_{ECRO}(\theta)+(1-\gamma)\hat{\mathcal{L}}_{CC}(\theta)\label{eq:jointLoss}
\end{equation}

The training pipeline for this algorithm can be seen in figure \ref{fig:TSCovpipeline}. It closely mirrors the structure of the \ECRO{} algorithm, with additional crucial steps to compute the necessary components of the loss presented in \ref{eq:jointLoss}. Within each epoch, the predicted uncertainty set $\U_\theta$ serves two purposes: i) Optimizing CRO to find the optimal policy $x^*(\cdot,\U_\theta)$ and assessing its associated risk; and ii) producing the binary variable $y(\psi,\xi,\U_\theta)$, which regression leading to $g_{\phi^*(\U_\theta)}(\cdot)$ serves to quantify the quality of the conditional coverage. 
The sum of task losses produce $\mathcal{L}_{DT}(\theta)$, which can be differentiated using decision-focused learning methods. The regression model $g_\phi(\psi)$ take the form of a feed-forward neural network with a sigmoid activation in the final layer and optimized using stochastic gradient descent. Algorithm \ref{alg:dualecrotraining} in appendix \ref{algos} presents the details of this \EECRO{} algorithm.

\begin{remark}
It is to be noted that achieving distribution-free finite sample conditional coverage guarantees is known to be impossible in the conformal prediction literature (see \cite{barber:limitsCondConf}). Recently, some progress has been made towards partial forms of conditional coverage guarantees (see \cite{gibbs2023conformal}) yet it is unclear what are the implications of exploiting such partial coverage properties for the downstream CRO decisions. It is also unclear how such conditional conformal prediction procedures could be integrated within and end-to-end CRO approach.
\end{remark}
\begin{figure*}  [htbp]
\centering
\subfloat[]{
  \includegraphics[width=0.34\textwidth]{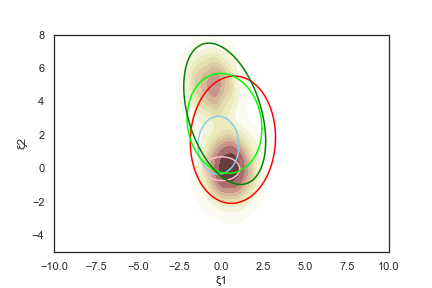}
}\hspace{-1.5em}%
\subfloat[]{
  \includegraphics[width=0.34\textwidth]{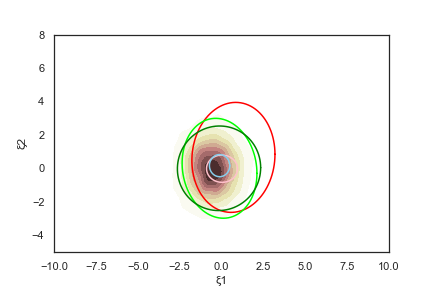}
}\hspace{-1.5em}%
\subfloat[]{
  \includegraphics[width=0.34\textwidth]{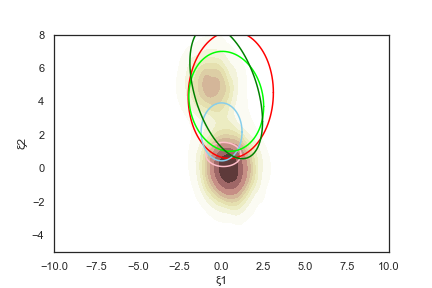}
}\\
\legendsquare{lgreen}~ \begin{tiny} \CROCS{} \end{tiny} \quad 
 \legendsquare{dgreen}~ \begin{tiny} \CROCCS{} \end{tiny} \quad  
\legendsquare{lblue}~\begin{tiny} \CROES{} \end{tiny}\quad
\legendsquare{lred}~ \begin{tiny} \ECRO{} \end{tiny} \quad
\legendsquare{rred}~ \begin{tiny} \EECRO{} \end{tiny}
\caption{Comparison of uncertainty set ($\alpha$ = 0.9) coverage for different $\psi$ realizations: (a) $[2.5, -0.2]^T$, (b) $[-2.6, 0.5]^T$, (c) $[2.7, 1.9]^T$. The shade indicate the true conditional distribution.\label{plot:covcomp}}\vskip -0.2in
\end{figure*}
\section{Experiments}
This section outlines our experimental framework devised to demonstrate the advantages of the ECRO method in learning the uncertainty sets tailored to covariate information. Our focus lies in assessing the utility of the model in: i) improving the CRO performance; and ii) achieving conditional coverage. We conduct a comparative analysis between our two end-to-end approaches, \ECRO{} and \EECRO{}, and three state-of-the-art ETO approaches to formulate contextual ellipsoidal sets. We first consider a Distribution-based contextual ellipsoidal uncertainty Set (\CROES{}) recently introduced in \cite{blanquero2023contextual}, where the conditional distribution of $\xi$ given $\psi$ is presumed to follow a multivariate normal distribution. Additionally, we explore two distributional-free approaches. A vanilla Conformal Prediction Set (\CROCS{}) uses conformal prediction on the output of a point predictor for $\xi$ given $\psi$, after shaping the ellipsoid (through an invariant $\Sigma$) using the residual errors \cite{johnstone2021conformal}. An Adapted version of Conformal Prediction Set (\CROCCS{}) (proposed in \cite{messoudi2022ellipsoidal}) adapts the shape $\Sigma$ using local averaging around the observed $\psi$. 

\begin{figure*}[!ht]
\centering
\subfloat[\begin{tiny}2017\end{tiny}]{
  \includegraphics[width=0.34\textwidth]{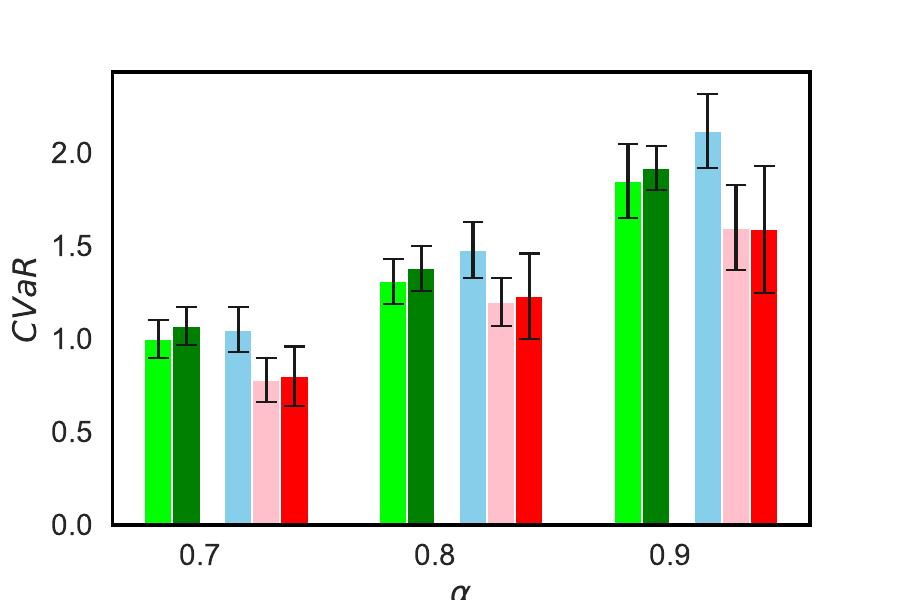}
}\hspace{-1.5em}%
\subfloat[\begin{tiny}2018\end{tiny}]{
  \includegraphics[width=0.34\textwidth]{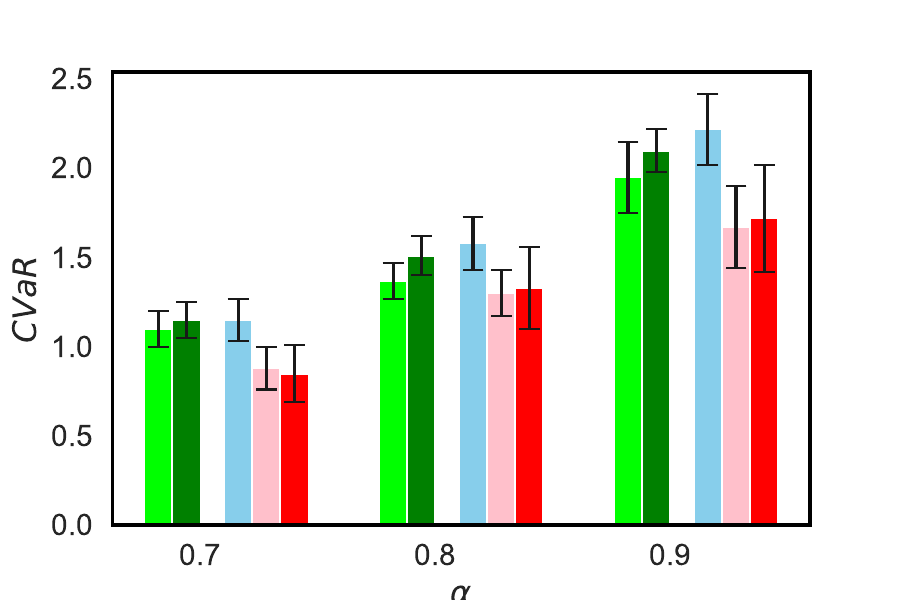}
}\hspace{-1.5em}%
\subfloat[\begin{tiny}2019\end{tiny}]{
  \includegraphics[width=0.34\textwidth]{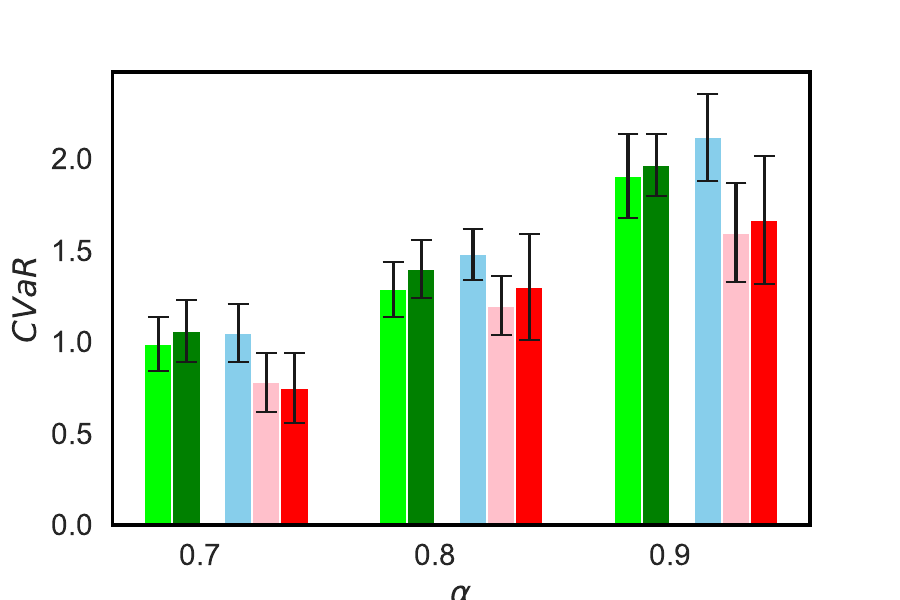}
}\\
\legendsquare{lgreen}~ \begin{tiny} \CROCS{} \end{tiny} \quad 
 \legendsquare{dgreen}~ \begin{tiny} \CROCCS{} \end{tiny} \quad  
\legendsquare{lblue}~\begin{tiny} \CROES{} \end{tiny}\quad
\legendsquare{lred}~ \begin{tiny} \ECRO{} \end{tiny} \quad
\legendsquare{rred}~ \begin{tiny} \EECRO{} \end{tiny}
\caption{Avg. CVaR of returns across 10 portfolio trajectory simulations. Error bars report 95\% CI.}\label{results:portfolio}
\end{figure*}
\subsection{The portfolio optimization application}

We explore the effectiveness of proposed methodologies in addressing a classic robust portfolio optimization problem. In this context, we define the cost function $c(x,\xi)$ as $-\xi^Tx$, where $x$ represents a portfolio comprising investments in $m$ different assets, with their respective returns denoted in the random vector $\xi$. Additionally, we impose constraints on $x$, encapsulated within $\mathcal{X}$, defined as $\mathcal{X} := \{ x \in \Re^m | \sum_{i = 1}^{m}x_i = 1, x \geq 0\}$. For this cost function, we obtain the partial concave conjugate function:
\begin{align}
\bar{c}_*(x, v) = \inf_{\xi : \|\xi\|_2\leq R}  v^T \xi - \xi^T x= -R\|v-x\|_2
\end{align}
Thus leading to problem \ref{eq:xopt} becoming
\begin{equation}
\begin{aligned}
\label{eq:OptPortfolioreformed2}
&\min_{x \in \mathcal{X}} f(x,\psi):=x^T\mu(\psi)+  \sqrt{x^T \Sigma(\psi)^{-1} x}
\end{aligned}
\end{equation}
when $R\rightarrow \infty$, thus capturing $\Xi:=\Re^m$.

\subsection{CRO performance using synthetic data}


We first consider a simple synthetic experiment environment where $m=2$ and where the pair $(\psi,\xi)$ is drawn from a mixture of three 4-d multivariate normal distributions. 
We sample N= 2000 observations and use 600 observations to train, 400 as validation, and 1000 observations for testing. All our results present statistics that are based on 10 simulations, each of which employed a slightly modified mixture model (see \href{https://anonymous.4open.science/r/End-to-end-CRO-513E/README.md}{github repository} for details). The \ECRO{} and \EECRO{} algorithms leverage deep neural networks with the corresponding task losses to learn the necessary components ($\mu_\theta(\psi),\Sigma_\theta(\psi)$) of $\mathcal{U}_\theta(\psi)$. 
All sets are calibrated for a probability coverage of 90\% and the risk of decisions is measured using CVaR at risk level $\alpha=0.9$. The average CVaR objective values and marginal coverages of the uncertainty sets can be found in the table \ref{tab:cvarsample}. 
One can notice that the end-to-end based methods, \ECRO{} and \EECRO{} significantly outperform the ETO methods on the CVaR performance. It appears that in order to maintain the required marginal coverage, the ETO approaches learned sets that resulted in overly conservative RO solutions. 

\begin{figure}[!t]
\centering
{
  \includegraphics[width=0.7\textwidth]{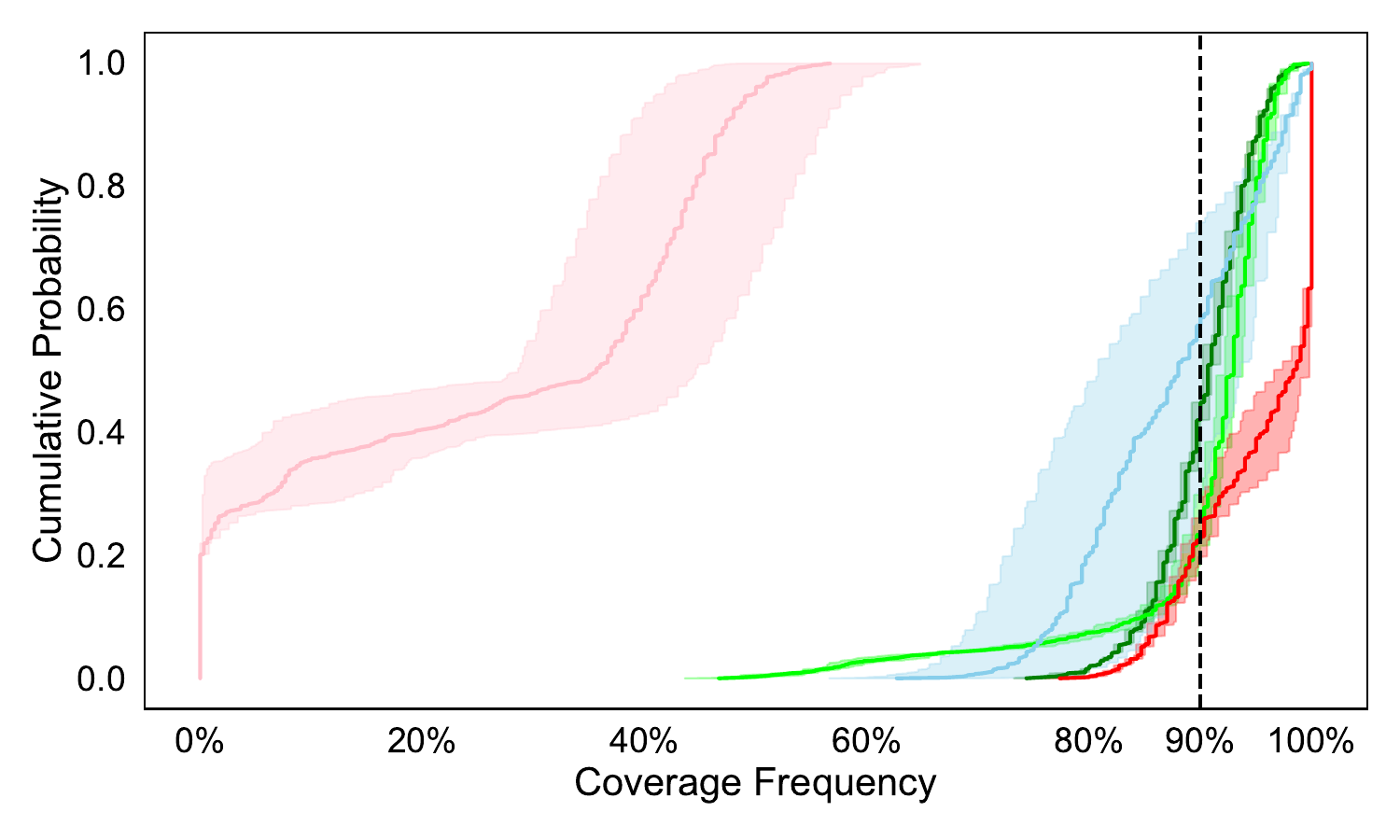}\\
}
\legendsquare{lgreen}~ \begin{tiny} \CROCS{} \end{tiny} \quad 
 \legendsquare{dgreen}~ \begin{tiny} \CROCCS{} \end{tiny} \quad  
\legendsquare{lblue}~\begin{tiny} \CROES{} \end{tiny}\quad
\legendsquare{lred}~ \begin{tiny} \ECRO{} \end{tiny} \quad
\legendsquare{rred}~ \begin{tiny} \EECRO{} \end{tiny}
\caption{Average cumulative distribution of conditional coverage frequency when $\psi$ is sampled uniformly from dataset over 10 simulated environments. Shaded region represent 90\% CI}\label{plot:cumplot}\vskip -0.1in
\end{figure}

\begin{table}[!ht]
\begin{center}
\scalebox{0.85}{
\begin{sc}
\begin{tabular}{lcccr}
\toprule
        & CVaR  &  Marginal Coverage \\
        \midrule
        \CROCS{} & $1.58 \pm 0.03$ & $91 \pm 1.8\%$   \\
        
        \CROCCS{} & $1.69 \pm 0.05$ & $91 \pm  1.4\%$  \\
        
        \CROES{} & $1.64 \pm 0.07$  &$85 \pm  7.8\%$ \\
        
        \ECRO{} & $1.03 \pm 0.10$ & $23 \pm  6.1\%$ \\
        
        \EECRO{} & $1.08 \pm 0.13$ &$92\pm  1.5\% $\\
        
\bottomrule
\end{tabular}
\end{sc}
}
\end{center}
\vskip -0.1in
\caption{Avg. CVaR and marginal coverage for $\alpha =1-\epsilon= 0.9$ over 10 simulated environments, error represent 90\% CI. \label{tab:cvarsample}}
\vskip -0.2in
\end{table}

Additionally, all the models except \ECRO{} appear to have the marginal coverage $~ 90\%$ which corresponds to the $\alpha$ level they are trained for. By disregarding the aspect of coverage, \ECRO{} was able to improve on the CVaR task but suffers poorly when it comes to coverage. Comparatively, the dual task based approach \EECRO{} 
was able to improve on the CVaR performance over the ETO approaches while still maintaining the necessary coverage.

As pointed out earlier, conditional coverage is a highly desirable property. Given that a synthetic environment gives us access to exact measurements of conditional coverage,  
figure \ref{plot:cumplot} presents the cumulative distribution of the observed conditional coverage frequencies when $\psi$ is sampled uniformly from the data set.
One can notice from the plot that \CROES{}, despite being closer to required marginal coverage, failed to provide accurate conditional coverage. Among the methods that use conformality score to calibrate the radius, \CROCCS{} method which uses localized covariance matrices has better conditional coverage. However, this comes at the price of CVaR performance. The advantage of the dual task-based approach, \EECRO{}, over the single task one are obvious. While \EECRO{} appears to have overshot the coverage compared to \CROCCS, which aligns closer to 90\%,  we argue that this is not an issue as it ends up providing more coverage than needed while generating nearly the best average CVaR value. 
In figure \ref{plot:covcomp} which overlays the various sets learned on the conditional distribution of $\xi$, one can notice that the sets adapt to the covariate information $\psi$ to provide the necessary conditional coverage.

\subsection{CRO using US stock data}\label{portfolioexp}

\textbf{Dataset and experimental Design} We follow the experimental design methodology proposed in \cite{chenreddy2022data}. Our experiments utilize historical US stock market data, comprising adjusted daily closing prices for 70 stocks across 8 sectors from January 1, 2012, to December 31, 2020, obtained via the Yahoo!Finance's API. Each year contains 252 data points, and we calculate percentage gain/loss relative to the previous day to construct our dataset, denoted as $\xi$. We incorporate trading volume of individual stocks and other market indices
as covariates. We test the robustness of all the models performance by solving the portfolio optimization problem on randomly selected stock subsets across different time spans. Utilizing 15 stocks in each window, we run the experiment ten times over three moving time frames. We maintain consistent parameters (learning rate $lr$, number of epochs $T$, step size $K$, $\gamma$). Further implementation and parameter tuning details can be found in Appendix \ref{sec:portfolio}. Figure \ref{results:portfolio} compares the avg. CVaR of returns and Table \ref{tab:marginal_coverage} presents the marginal coverage across difference confidence levels for models.

It is evident from the CVaR comparison that the task based methods \ECRO{} and \EECRO{} consistently performs better over the ETO models. Among ECRO approaches, we can clearly observe an advantage for \EECRO{}  over  \ECRO{}, which has on par CVaR performance while having out of sample marginal coverage closer to the expected target level. Conformal-based ETO methods have a good marginal coverage as they are designed to have the desired coverage. 
Especially, \CROCCS{} and \CROCS{}, being calibrated using conformal prediction which produce statistically valid prediction regions have near target coverage levels. 

\begin{table}[!ht]
\begin{center}
\scalebox{0.8}{
\begin{sc}
\begin{tabular}{lclll}
\hline
Model & \multicolumn{1}{l}{Year} & \multicolumn{3}{c}{Marginal cov. (\%)} \\ \cline{3-5} 
 &  & \multicolumn{3}{c}{Target $1-\epsilon$} \\ 
 & \multicolumn{1}{l}{} & 70\% & 80\% & 90\% \\ \hline
\CROCS{} & \multirow{5}{*}{2017} & 68 & 78 & 87 \\
\CROCCS{} &  & 68 & 77 & 89 \\
\CROES{} &  & 54 & 72 & 85 \\
\ECRO{} &  & 22 & 26 & 28 \\
\textbf{\EECRO{}} &  & \textbf{72} & \textbf{79} & \textbf{88} \\ \hline
\CROCS{} & \multirow{5}{*}{2018} & 67 & 79 & 88 \\
\CROCCS{} &  & 68 & 78 & 87 \\
\CROES{} &  & 59 & 75 & 87 \\
\ECRO{} &  & 23 & 24 & 29 \\
\textbf{\EECRO{}} &  & \textbf{71} & \textbf{80} & \textbf{93} \\ \hline
\CROCS{} & \multirow{5}{*}{2019} & 69 & 78 & 88 \\
\CROCCS{} &  & 71 & 78 & 89 \\
\CROES{} &  & 61 & 76 & 86 \\
\ECRO{} &  & 26 & 30 & 32 \\
\textbf{\EECRO{}} &  & \textbf{69} & \textbf{78}& \textbf{92} \\ \hline
\end{tabular}
\end{sc}
}
\end{center}
\caption{Marginal Coverage}
\label{tab:marginal_coverage}
\end{table}

\section{Conclusion} 
In summary, the paper introduces a novel framework for conditional robust optimization by combining machine learning and optimization techniques in an end-to-end approach. The study focuses on enhancing the conditional coverage of uncertainty sets and improving CRO performance. Through comparative analysis and simulated experiments, the proposed methodologies show superior results in robust portfolio optimization. The findings point to the importance of uncertainty quantification and highlight the effectiveness of an end-to-end approach in risk averse decision-making under uncertainty.

\bibliography{references}
\clearpage
\appendix
\section{Algorithms} \label{algos}

\begin{algorithm}[!ht]
  \caption{Dual ECRO Training with Trust Region Solver}\label{alg:dualecrotraining}
  \begin{algorithmic}[1]
  \State \textbf{input}: dataset $\Dxipsi$, max epochs $T$, max TRO steps $K$, batch size $N$, protection level $\alpha$
  \State Initialize a warm start  buffer $\{\bar{x}_1,\dots,\bar{x}_M\}$ with each $\bar{x}_i\in\mathcal{X}(\psi_i)$
  \State Initialize network parameters $\theta$ and $t=1$
      \While{not converged and ($t \le T$)}
        \State Sample a batch of $N$ indices $\mathcal{B}\subset \{1,\dots,M\}$
         \For{$i\in \mathcal{B}$}
         \State //Run TRO for up to $K$ steps
        \State $x_i^t, \lambda_i^t, \nu_i^t \gets$ \textsc{TRO}($\bar{x}_{i}$, $\mu_{\theta}(\psi_{i}),\Sigma_\theta(\psi_{i}), K)$
        \State $\bar{x}_i\gets x_i^t$ \Comment{Update warm start}
        \State $y_i^t \gets \1\{\xi_i \in \mathcal{E}(\mu_\theta(\psi_i),\Sigma_\theta(\psi_i))\}$
      \EndFor
      \State $\phi^t \gets$ \textbf{solve} prob \eqref{eq:NLL} for $\{(\psi_i,y_i^t)\}_{i\in\mathcal{B}}$
        \State \textbf{compute} $\mathcal{L}_{DT}(\theta)$ and $\nabla_\theta\mathcal{L}_{DT}(\theta)$ for $i\sim\mathcal{B}$
        \State $\theta \gets \theta - \mbox{step size}\cdot\nabla_\theta\mathcal{L}_{DT}(\theta)$ 
        
      \EndWhile
      
      \State \textbf{return} $\theta$
  \end{algorithmic}
\end{algorithm}

\section{Supplementary for Experiments}
\subsection{Convergence comparison} \label{fig:covcomp}
\begin{figure}[ht]
\centering
  \includegraphics[width=0.7\textwidth]{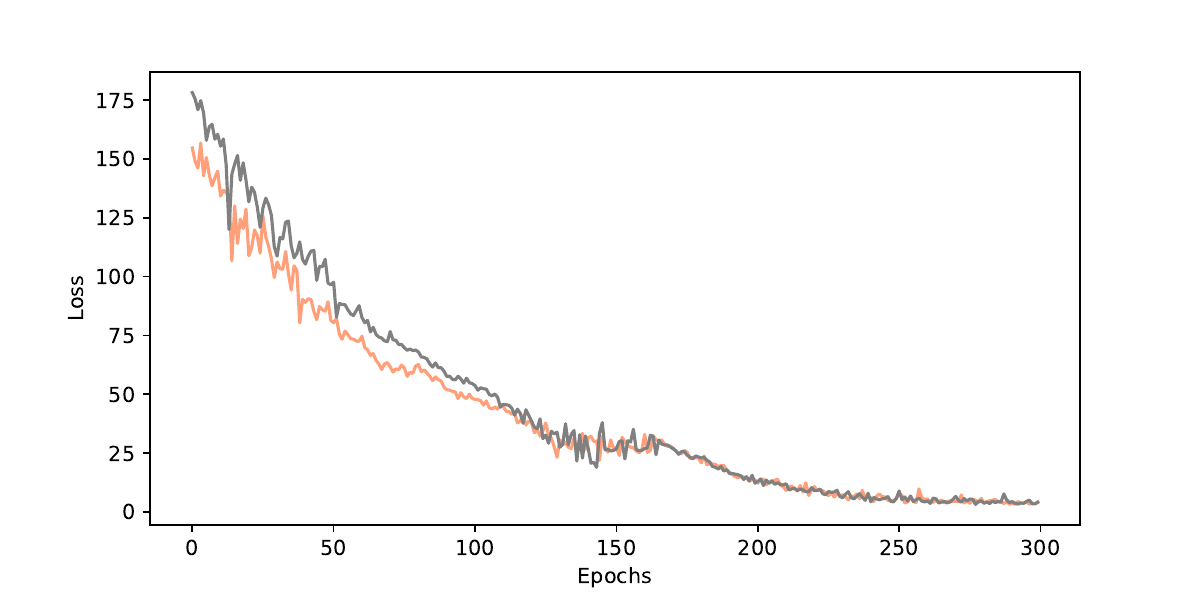}
\\
\legendsquare{grey}~ \begin{tiny} $5$-steps TRO\end{tiny} \quad 
 \legendsquare{salmon}~ \begin{tiny} TRO \end{tiny} 
\caption{Convergence comparison between $5$-steps TRO (46 min) and full TRO (129 min).\label{results:modelcomp}}
\vskip -0.3in
\end{figure}
\subsection{Synthetic conditional data generation}\label{sec:synth_data_gen}
As we have access to the paramters of the simulation environment, we can sample the  conditional multivariate Gaussian distribution of $\xi$ upon observing $\psi$  as:
\begin{align*}
&\mu(\xi)|\psi = \boldsymbol{\mu}_{\xi} + \mathbf{\Sigma}_{\xi\psi} \cdot \mathbf{\Sigma}_{\psi\psi}^{-1} \cdot (\psi - \boldsymbol{\mu}_{\psi}) \\
&\Sigma(\xi)|\psi = \mathbf{\Sigma}_{\xi\xi} - \mathbf{\Sigma}_{\xi\psi} \cdot (\mathbf{\Sigma}_{\psi\psi})^{-1} \cdot \mathbf{\Sigma}_{\psi\xi}
\end{align*}
where $\boldsymbol{\mu}_{\xi}$ is the mean vector of the dependent variables. $\boldsymbol{\mu}_{\psi}$ is the mean vector of the independent variables. $\mathbf{\Sigma}_{\xi\xi}$ is the covariance matrix of the dependent variables. $\mathbf{\Sigma}_{\xi\psi}$ is the cross-covariance matrix between dependent and independent variables. $\mathbf{\Sigma}_{\psi\psi}$ is the covariance matrix of the independent variables, and $\mathbf{\psi}$ is the observed independent variables. We sample data from these conditional distributions $\xi|\psi \sim \mathcal{N}(\mu_{\psi}(\xi), \Sigma_{\psi}(\xi))$ and compare the coverage of these observations.

\subsection{Parameter tuning procedure}\label{sec:portfolio}
In this section, we explore the parameter tuning methodology applied to train the network introduced in Section \ref{portfolioexp}. Given the time series nature of the data, we employ a rolling window technique for network training. Our architecture depends on a set of hyperparameters, defined as follows: \(lr\) for learning rate, \(T\) for the maximum number of epochs, \(K\) for the maximum TRO steps, \(B\) for the batch size, and \(\alpha\) for the target level. We partition the data into training and validation periods and examine the optimal combination through grid search. For each combination, we train the network and derive the optimal policy using the training data, then applying it to the unseen validation data. The optimal combination is selected based on the lowest CVaR on the validation dataset, viewing this as a worst-case return minimization problem.

Regarding the \EECRO{} algorithm, which balances between two losses—the CRO objective and the conditional coverage loss—we follow a specific strategy to identify the best performing model. At each epoch, we save the model and initiate model selection only after achieving the required training coverage. Subsequently, we retain the best models meeting the coverage criteria until convergence conditions are met. Among all saved models meeting the coverage requirement, we choose the one with the best CVaR objective.

\newpage
\subsection{Architecture} \label{architecture}
\begin{figure}[!ht]
  \centering
  \includegraphics[width=68mm]{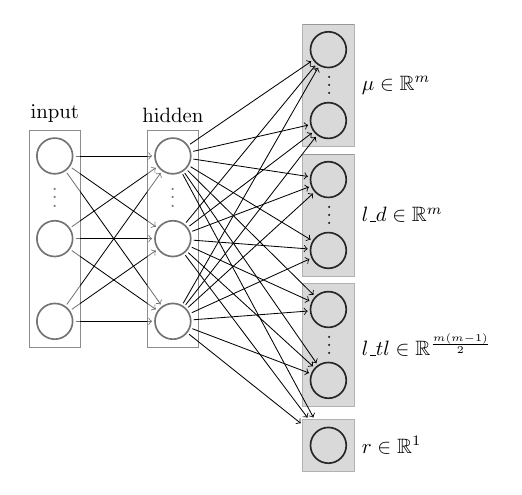}
  \caption{Example Neural Network.}
  \label{fig:nn}
\end{figure}

\end{document}

%% file: math_commands.tex
\usepackage{amsmath,amsfonts,bm}

\def\eqref#1{equation~\ref{#1}}

\def\1{\bm{1}}

\DeclareMathAlphabet{\mathsfit}{\encodingdefault}{\sfdefault}{m}{sl}
\SetMathAlphabet{\mathsfit}{bold}{\encodingdefault}{\sfdefault}{bx}{n}

\DeclareMathOperator*{\argmax}{arg\,max}
\DeclareMathOperator*{\argmin}{arg\,min}